\documentclass{SCIS2025}
\usepackage{times}  
\usepackage{helvet}  
\usepackage{courier}  
\usepackage{graphicx} 
\usepackage{caption} 
\usepackage{algorithm}
\usepackage{algorithmic}
\usepackage{newfloat}
\usepackage{listings}
\usepackage{url}            
\usepackage{booktabs}       
\usepackage{graphicx}
\usepackage{amsfonts}       
\usepackage{nicefrac}       
\usepackage{microtype}      
\usepackage{xcolor}         
\usepackage{multirow}
\usepackage{multicol}
\usepackage{bm}
\usepackage{algorithm}
\usepackage{algorithmic}
\usepackage{amsmath}
\usepackage{cleveref}
\usepackage{caption}
\begin{document}
\ArticleType{RESEARCH PAPER}
\Year{2025}
\Month{}
\Vol{68}
\No{}
\DOI{}
\ArtNo{000000}
\ReceiveDate{}
\ReviseDate{}
\AcceptDate{}
\OnlineDate{}
\title{A Smooth Transition Between Induction and Deduction: Fast Abductive Learning Based on Probabilistic Symbol Perception}{A Smooth Transition Between Induction and Deduction: Fast Abductive Learning Based on Probabilistic Symbol Perception}
\author[]{Lin-Han Jia}{}
\author[]{Si-Yu Han}{}
\author[]{Lan-Zhe Guo}{{guolz@lamda.nju.edu.cn}}
\author[]{Zhi Zhou}{}
\author[]{Zhao-Long Li}{}
\author[]{Yu-Feng Li}{{liyf@lamda.nju.edu.cn}}
\author[]{Zhi-Hua Zhou}{}
\AuthorMark{Lin-Han Jia}
\AuthorCitation{Lin-Han Jia, Si-Yu Han, Lan-Zhe Guo, Zhi Zhou, Zhao-Long Li, Yu-Feng Li and Zhi-Hua Zhou}
\address[]{National Key Laboratory for Novel Software Technology, Nanjing University, Nanjing 210023, China}
\abstract{Abductive learning (ABL) that integrates strengths of machine learning and logical reasoning to improve the learning generalization, has been recently shown effective. However, its efficiency is affected by the transition between numerical induction and symbolical deduction, leading to high computational costs in the worst-case scenario. Efforts on this issue remain to be limited. In this paper, we identified three reasons why previous optimization algorithms for ABL were not effective: insufficient utilization of prediction, symbol relationships, and accumulated experience in successful abductive processes, resulting in redundant calculations to the knowledge base. To address these challenges, we introduce an optimization algorithm named as Probabilistic Symbol Perception (PSP), which makes a smooth transition between induction and deduction and keeps the correctness of ABL unchanged. We leverage probability as a bridge and present an efficient data structure, achieving the transfer from a continuous probability sequence to discrete Boolean sequences with low computational complexity. Experiments demonstrate the promising results.}
\maketitle
\section{Introduction}
The relationship between empirical inductive learning and rational deductive logic has been a frequently discussed philosophical question throughout human history. In recent years, this issue has also emerged in the field of artificial intelligence  (AI), manifesting as the challenge of integrating machine learning and logical reasoning, two relatively independent technologies.

Many efforts have focused on this integration issue.
Neuro-symbolic (NeSy) learning ~\cite{mao2019neuro} proposes to enhance neural networks with symbolic reasoning. However, it requires lots of labeled training data and is difficult to extrapolate. Probabilistic Logic Program (PLP) ~\cite{de2015probabilistic} is a heavy-reasoning light-learning way because the most workload is to be finished by logical reasoning though some elements of machine learning are introduced. Statistical Relational Learning (SRL) ~\cite{getoor2007introduction} is a heavy-learning light-reasoning way in opposite. DeepProbLog ~\cite{manhaeve2018deepproblog}, which unifies probabilistic logical inference with neural networks but with exponential complexity of probabilistic distributions on the Herbrand base.
 
In order to better integrate the advantages of both fields, Abductive Learning (ABL) ~\cite{zhou2019abductive} is introduced to allow to infer labels that are consistent with some prior knowledge by reasoning over high-level concepts. It is a recent generic and effective framework that bridges any kind of machine learning algorithms and logical reasoning by minimizing the inconsistency between the pseudo labels obtained from machine learning and logical reasoning. The inconsistency value is calculated by a designed distance function.

However, the efficiency in previous ABL studies as well as NeSy approaches is affected by the transition between numerical induction and symbolical deduction, leading to high computational costs in the worst-case scenario. A smooth transition to bridge the calculations in two fields will be important. In ABL, the logic reasoning module takes the unreliable parts of the symbols predicted by the machine learning model as variables, while treating the reliable parts as constants. Connecting the rules in the knowledge base, it then uses the reliable parts to infer the corrected symbols for the unreliable parts and feeds these back to the machine learning model for updating. 
Therefore, this can be seen as an \textit{optimization problem}, where the optimization variables are Boolean variables that determine which predicted symbols should be variables and the others should be constants in reasoning. The optimization objective is the number of final correct logical reasoning results. In fact, the essence of this optimization problem is to simulate human meta-reasoning ability based on symbol systems, which is also called symbol sensitivity, or abstract perception.

Due to the fact that this optimization problem involves both numerical and symbolical values, there is still a lack of efficient optimization techniques. It is worth mentioning that throughout the optimization process, the number of executions in machine learning does not correspond one-to-one mapping with the number of logical reasoning executions, but rather \textit{one-to-many}, resulting in logical reasoning occupying most of the time, as shown in \cref{Abductive}, which typically requires hundreds of logical reasoning to correspond to a single machine learning process. To this end, significantly reducing the number of attempts on Boolean variables in logical reasoning denoted as $T_{ac}$ can help improve the efficiency of ABL and potentially even enhance performance.
\begin{figure*}[htbp]
\vskip 0.2in
\begin{center}
\centerline{\includegraphics[scale=0.45]{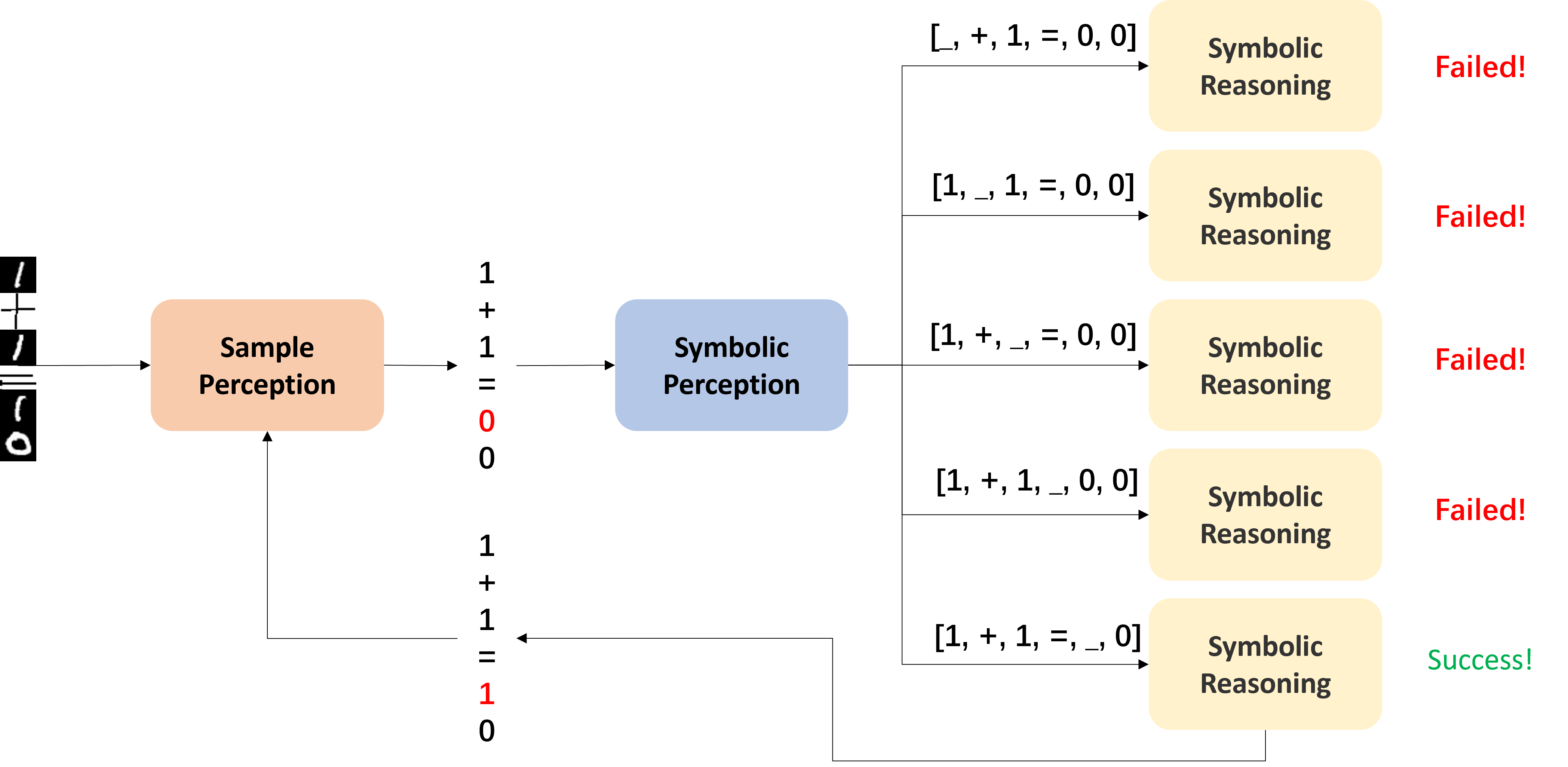}}
\caption{The process diagram of ABL reveals a one-to-many relationship between the inductive learning and logical reasoning modules.}
\label{Abductive}
\end{center}
\vskip -0.2in
\end{figure*}

In order to address this optimization problem, ABL used general gradient-free optimization methods, such as search-based algorithms POSS ~\cite{NIPS2015_b4d168b4, Liu_2022} and sample-based algorithms RACOS ~\cite {yu2016derivative}. All of these are different from the way human experts combine induction and reasoning, and perform many ineffective operations, leading to resource waste. Specifically, compared to human experts, the reasons for the inefficiency of existing algorithms are related to three aspects: starting in a random state without utilizing predictive information, failing to utilize the relationships between symbols, and performing zero experience optimization every time without accumulating successful ABL process experience.

We propose an optimization algorithm for ABL named Probabilistic Symbol Perception (PSP) that can respond quickly with few trial and error attempts. The intermediate process from machine learning to logical reasoning inevitably involves a transition from continuous to discrete. Probability is a natural tool because its values are continuous, but it represents the degree to which each discrete event occurs, making it the smoothest transition. Therefore, our algorithm consists of two steps: the first step involves a neural network capable of processing sequences, which takes probability predictions from machine learning models as input. It combines the structural information of the symbol sequence and outputs the probability that each position in the output sequence needs to be corrected as a variable. The second step involves an algorithm for generating $T_{ac}$ discrete Boolean sequences from the continuous probability sequence, determining which symbols are variables in logical reasoning. In the case of a sequence length of $l$, the complexity of traditional calculation methods is $O(2^l\log 2^l)$ in the worst-case scenarios. In contrast, the proposed new algorithm only requires $O (l \log l+T_{ac} \log T_{ac}+lT_{ac})$ time complexity. Similar to traditional ABL, we keep the strict correctness to be unchanged. 
In addition, we pass the Boolean sequence corresponding to the optimal result to the sequence neural network, enabling it to accumulate successful experiences.
\textbf{Our contribution.} 
Our contribution can be summarized in three main aspects. First, we identified the core issue that leads to slow ABL optimization: reducing the number of attempts on Boolean variables in logical reasoning.

We also identified three main shortcomings of past optimization algorithms. Second, we propose a solution to alleviate these drawbacks with low complexity. Finally, experiments demonstrate the promising results of our proposal in terms of good accuracy with reduced attempts.


\section{Related Works}
\subsection{Integration of Perception and Reasoning} 
 %
Bridging machine learning and logical reasoning is a well-known holy grail problem in artificial intelligence. NeSy learning ~\cite{mao2019neuro} proposes to enhance machine learning with symbolic reasoning. It tries to learn the ability for both perception from the environment and reasoning from what has been perceived. However, it requires lots of labeled training data and is difficult to extrapolate. PLP ~\cite{de2015probabilistic} extends first-order logic (FOL) to accommodate probabilistic groundings and conduct probabilistic inference. SRL ~\cite{getoor2007introduction} tries to construct a probabilistic graphical model based on domain knowledge expressed in FOL. PLP and SRL are different from the way human beings solve problems as human beings can use perception and reasoning seamlessly while PLP and SRL focus on one side more. DeepProbLog ~\cite{manhaeve2018deepproblog} unifies probabilistic logical inference with neural network training by gradient descent. However, the probabilistic inference in these methods could be inefficient for complicated tasks because of exponential complexity. Formal Logic Deduction (FLD) ~\cite{morishita2023learning} tries to grant language models with reasoning abilities through logic system, where the language model is the main body of the system. Multi-layer perception mixer model (MLP-Mixer) ~\cite{amayuelas2022neural} is quite the opposite, utilizing neural networks to study logic.

Different from the works above, ABL tries to bridge machine learning and logical reasoning in a mutually beneficial way. Zhou first proposed the concept of ABL ~\cite{zhou2019abductive}. Dai and Zhou elaborated the framework of ABL and applied ABL to the task of recognizing handwritten formulas with good results ~\cite{dai2019bridging}. Huang developed a similarity-based consistency measure for abduction called ABLSim ~\cite{huang2021fast}, which takes the idea that samples in the same category are similar in feature space. Huang and Li presented an attempt called SS-ABL ~\cite{huang2020semi} which combines semi-supervised learning and abductive learning and applied it to theft judicial sentencing with good results. In order to alleviate the problem of low efficiency and high cost of the knowledge base in ABL, Huang presented ABL-KG ~\cite{huang2023enabling} which enables abductive learning to exploit general knowledge graph. In recent years, this field has made some research advancements.
\subsection{Derivative-Free Optimization} 
Most of the optimization algorithms used in ABL are Derivative-Free Optimization. Derivative-free optimization algorithms are a class of optimization methods that do not rely on gradient information to find the minimum or maximum of a function. RACOS ~\cite{yu2016derivative} is a proposed classification-based derivative-free optimization algorithm. Unlike other derivative-free optimization algorithms, the sampling region of RACOS is learned by a simple classifier. Two improving methods mentioned in ZOOpt are SRACOS ~\cite{hu2017sequential} and ASRACOS ~\cite{liu2019asynchronous},  respectively are the sequential and asynchronous versions of RACOS. POSS ~\cite{NIPS2015_b4d168b4} is another derivative-free optimization approach that employs evolutionary Pareto optimization to find a small-sized subset with good performance. POSS treats the subset selection task as a bi-objective optimization problem that simultaneously optimizes some given criterion and the subset size. POSS has been proven with the best so far approximation quality on these problems. PPOSS ~\cite{qian2016parallel} is the parallel version of the POSS algorithm.
Yu released a toolbox ZOOpt ~\cite{Liu_2022} with effective derivative-free optimization algorithms.

\section{Problem Setting and Formulation}
\subsection{The Machine Learning Module}
In ABL, we are given a set of $L$ training labeled data $D=\{(X_1,Y_1),(X_2,Y_2),\dots,(X_L,Y_L)\}$, where $X_i=[x_{i1},x_{i2},\dots,x_{il_i}],x_{ij}\in \mathcal{R}^d$ represents a sample sequence of length $l_i$, with $d$ denoting the feature dimension. Each sample corresponds to a symbol sequence of length $l_i$, denoted as $S_i=[s_{i1},s_{i2},\dots,s_{il_i}],s_{ij}\in SYM$, where $SYM$ represents the set of all symbols. However, the true symbol sequence is unknown and needs to be predicted jointly through machine learning and logical reasoning. The label $Y_i\in \{False, True\}$ serves as a Boolean variable indicating whether the symbol sequence $S_i$ corresponding to the sample sequence $X_i$ complies with the logic.

During ABL, we need to train a perception model, denoted as $g$, for symbol induction, which classifies samples into corresponding symbols. Typically, the perception model outputs the probabilities of a sample belonging to different symbols. This yields a probability sequence $P_i=[p_{i1},p_{i2},\dots,p_{il_i}]$, where $p_{ij}=g(x_{ij})$, and $p_{ij}\in\mathcal{R}^{|SYM|}$, with $|SYM|$ representing the number of all symbols. Based on the probability sequence, we obtain the prediction sequence $O_i=[o_{i1},o_{i2},\dots,o_{il_i}]$ for the sample sequence using a machine learning classifier based on the perception model. Here, $o_{ij}=f(x_{ij})=\arg\max_k p_{ij}[k]$, where $f$ represents the machine learning classifier obtained from the perception model $g$.
\subsection{The Logical Reasoning Module}
The logical reasoning module contains a knowledge base (KB) used to store logical rules formed from human knowledge. Its function is to deduce the values of variables based on inputs of both variables and constants, using the constants and rules stored in the KB. To integrate the machine learning module with the logical reasoning module, we need to determine whether the prediction sequence $O_i$ is compatible with the knowledge base KB. If the sequence is compatible, the prediction sequence does not need modification. However, in cases where it is not compatible, parts of the observation results in the prediction sequence $O_i$, which are all constants, need to be modified to variables. Thus, we denote a Boolean sequence $B_i=[b_{i1},b_{i2},\dots,b_{il_i}], b_{ij}\in\{False,True\}$, where if $b_{ij}=False$, it means keeping $o_{ij}$ as a constant, and if $b_{ij}=True$, it means replacing $o_{ij}$ with a variable, which is symbolized as `$\_$' in Prolog. For example, for a prediction sequence $O_i=[o_{i1},o_{i2},o_{i3},o_{i4},o_{i5}]$ of length 5, if we have the sequence $B_i=[False,True,False,True,False]$, then the observation sequence is replaced with ${O’}_i=[o_{i1},\_,o_{i3},\_,o_{i5}]$. By inputting this into the knowledge base KB, we obtain a new observation sequence $\Delta(O_i)=[o_{i1},\delta(o_{i2}),o_{i3},\delta(o_{i4}),o_{i5}]$, where $\delta(o_{i2})$ and $\delta(o_{i4})$ are the correction results obtained by combining $o_{i1}$, $o_{i3}$, and $o_{i5}$ with the KB through logical reasoning. Providing the corrected result $\Delta(O_i)$ to the machine learning model $f$ allows for further updating of the model, gradually integrating the knowledge from the knowledge base into the machine learning model. Let $\models$ denote logical entailment. $\psi$ represents the update to model $f$. The above process can be symbolized as:
\begin{align*}
    &(X_i, f) \triangleright O_i\\
    s.t. &(KB, O_i) \models Y_i, or (KB,\Delta(O_i))\models Y_i, f\leftarrow \psi(f,\Delta(O_i))
\end{align*}
\subsection{The Optimization Problem}
In the aforementioned ABL process, the procedures of machine learning and logical reasoning are predetermined, and the only uncertainty lies in the bridge between the two, which is the Boolean sequence $B$ used to determine whether symbols should be treated as constants or variables. The essence of selecting $B$ is to identify the current prediction results that are more likely to be erroneous and need correction. This selection process is crucial in the ABL process because if correctly predicted variables are designated as constants, logical reasoning may fail due to information loss. Conversely, if incorrectly predicted variables are designated as constants, even successful logical reasoning may lead to erroneous abduction. Therefore, ABL algorithms need to optimize the selection of the Boolean sequence $B$, with the optimization goal being the number of samples where the prediction results match the knowledge base, i.e.:
\begin{align*}
&\max_B\max_{D_c\subset D} |D_c|\\
s.t.&\forall (X_i,Y_i)\in D_c, (X_i, f) \triangleright O_i,\\&
(KB,B_i,O_i) \triangleright \Delta(O_i), (KB,\Delta(O_i))\models Y_i\\
\end{align*}
For such an optimization problem, the variable to be optimized is the Boolean sequence $B$, which constitutes a discrete optimization problem. Conventional numerical optimization algorithms cannot be applied to this problem; only gradient-free ones are suitable. Additionally, we observe that the search space for the variables in this problem grows exponentially with the sequence length. It is challenging to obtain the global optimal solution by traversing every possible combination.
\subsection{Time Consumption}
From the perspective of the entire ABL process, assuming the execution times of the machine learning module and the logical reasoning module are fixed, denoted as $T_{ml}$ and $T_{lr}$ respectively. 
We denote $T_{ac}$ as the number of times the knowledge base needs to be accessed for logical reasoning after each machine learning iteration. It deduces the total time $T_{total}$ consumed by one ABL process:
$T_{total}=T_{ml}+T_{ac}\cdot T_{lr}$. In normal situations, $T_{lr} >> T_{ml}$.
This makes the number of trial iterations, $T_{ac}$, the key factor in determining the total time consumption of ABL. Traditional optimization algorithms often set $T_{ac}$ to hundreds or more. However, such overhead is impractical in applications. Therefore, we are committed to addressing this challenge by completing ABL with a smaller $T_{ac}$.

\section{Methodology}
We propose an optimization algorithm for ABL that can respond quickly with very few trial-and-error attempts. Our algorithm mainly consists of two steps: In the first step, we achieved a smooth transition from sample perception to symbol perception, addressing the shortcomings of previous algorithms from three aspects. In the second step, we realized a smooth transition from symbol perception to symbol reasoning, proposing an efficient algorithm for converting the continuous probability sequence into the top $T_{ac}$ discrete Boolean sequences. 
\subsection{From Sample Perception to Symbolic Perception}
Previous optimization algorithms fail to perform well when the number of trial-and-error attempts is strictly limited. It can be attributed to their inability to effectively narrow down the search space. This limitation is primarily manifested in the following three aspects:
\begin{itemize}
    \item \textbf{Underutilization of Prediction:} Previous optimization algorithms often start at a random state and fail to leverage the information provided by the prediction.
    \item \textbf{Underutilization of Symbol Relationship:} They tend to overlook the relational information inherent in the structure of symbol sequences.
    \item \textbf{Underutilization to Accumulate Experience:} They cannot accumulate experience from successful abductive processes. Instead, they restart from scratch with zero experience each time they encounter new data.
\end{itemize}

Addressing these limitations is crucial for developing optimization algorithms that can perform well under finite conditions and efficiently handle symbol perception tasks.
\subsubsection{Utilization of Perceptual Information}
Machine learning models typically first produce probabilistic predictions, and then select the symbol with the highest probability as the definitive result. Previous ABL algorithms directly searched for suitable Boolean variables $B$ based on the absolute results, which resulted in a loss of a significant amount of perceptual information. In cases where samples are insufficient, the perceptual results of machine learning algorithms themselves are ambiguous and uncertain. Preserving and utilizing their probability information is advantageous for determining which symbol predictions are unreliable, and identifying symbols that may have been predicted incorrectly. Therefore, we use the probabilistic symbols obtained from the machine learning model as input to the symbol perception module. For a sample sequence $X_i=[x_{i1},x_{i2},\dots,x_{il_i}],x_{ij}\in \mathcal{R}^d$ and a machine model $g$, it's easy to obtain the probabilistic symbols $P_i=[p_{i1},p_{i2},\dots,p_{il_i}]$, where $p_{ij}=g(x_{ij}), p_{ij}\in\mathcal{R}^{|SYM|}$.
\subsubsection{Utilization of Symbol Structure}
\begin{figure*}[htbp]
\begin{center}
\centerline{\includegraphics[scale=0.47]{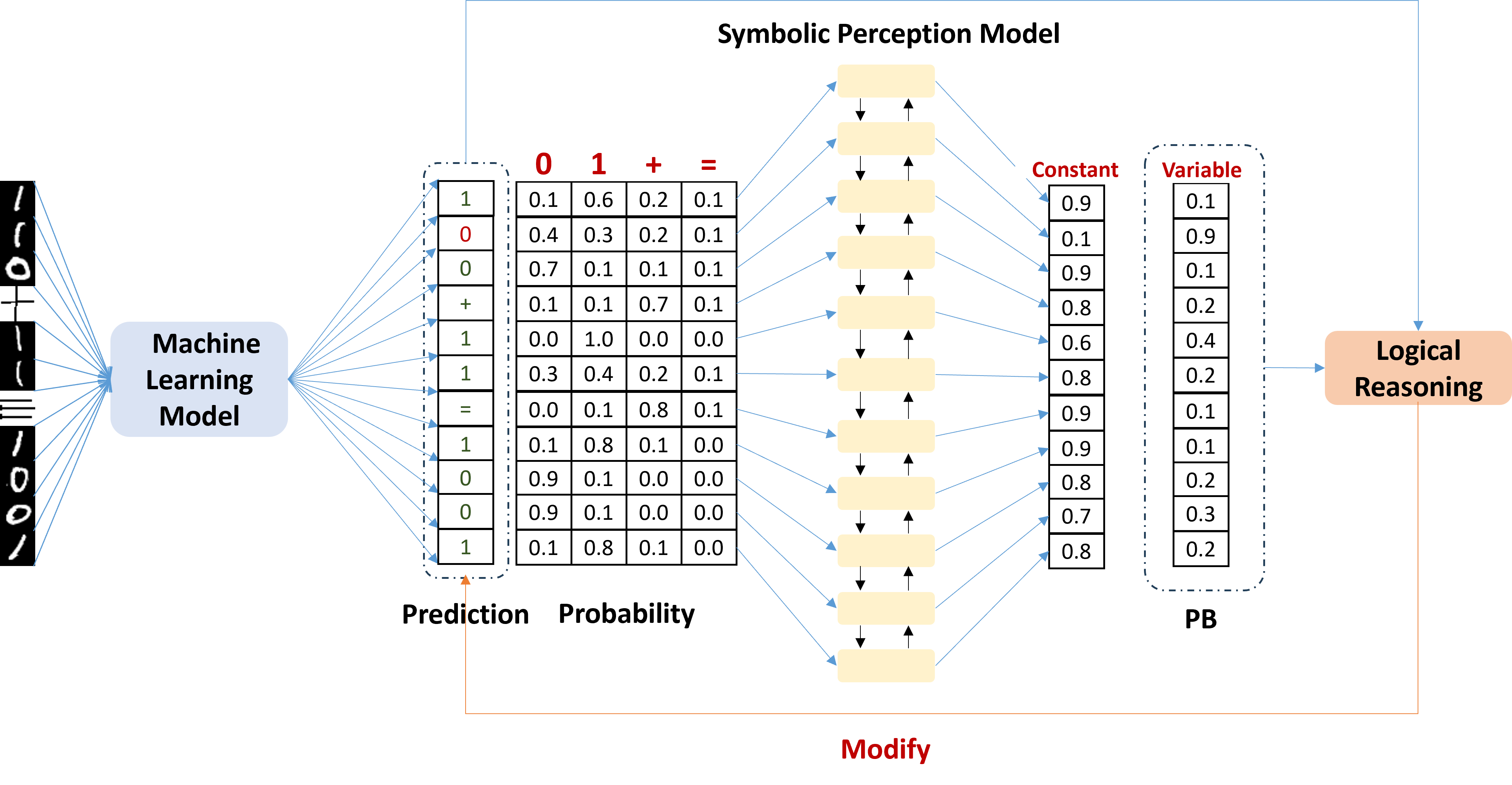}}
\caption{The schematic diagram of ABL transitioning from sample to symbolic perception.}
\label{Perception}
\end{center}
\end{figure*}
In the process of understanding symbols, the meaning of an individual symbol is typically limited and requires connections with other symbols to better understand its significance. Moreover, logical relationships are formed among multiple symbols, which are then induced into knowledge. Therefore, symbol-level perception is based on symbol sequences. At the sample perception level, we have obtained probabilistic symbol sequences through machine learning models. The subsequent requirement is to combine the probabilistic predictions of individual symbols from machine learning with the relationships between symbols within the symbol sequence to predict $B$, determining whether each symbol should be treated as a variable or a constant in the logical reasoning process as shown in \cref{Perception}. In current machine learning technologies, neural networks designed for sequence processing, such as bidirectional recurrent neural networks ~\cite{schuster1997bidirectional} and Transformer ~\cite{vaswani2017attention} architectures, are well-suited for this task. They can output the probability of each symbol being predicted correctly or incorrectly $PB_i=[pb_{i1},pb_{i2},\dots,pb_{il_i}]$, which can serve as the basis for obtaining $B$. For a bidirectional sequence neural network $BSNN$:
\begin{align*}[pb_{i1},pb_{i2},\dots,pb_{il_i}]=BSNN([p_{i1},p_{i2},\dots,p_{il_i}])
\end{align*}
\subsubsection{Accumulation of Experience}
Furthermore, neural networks designed for sequence processing can accumulate experience from each inference process, without wasting successful inference cases. After the inference process, we can use the Boolean variables that achieved the highest consistency between the knowledge base and the data during the evaluation as supervisory information. This information can be used to update the BSNN. Consequently, the successful experiences of inference can be effectively accumulated, enabling the neural network to possess symbol-based intuition akin to human scientists. This intuition plays a crucial role in significantly reducing unnecessary search iterations. 

More importantly, the sequence neural network can be pre-trained. This means that by considering only partially replacing symbols in originally logically correct sequences and allowing the sequence neural network to predict the positions to be replaced, we can accumulate sufficient experience very simply before the start of ABL. This process enables the model to converge faster and perform better.
\subsection{From Symbolic Perception to Symbolic Reasoning}
\subsubsection{From Continuous to Discrete}
Taking symbol perception into further consideration, we aim to obtain $T_{ac}$ Boolean variables $B$ for querying the knowledge base based on the results of symbol perception $PB$, the probability of each symbol being predicted correctly or incorrectly. Our objective is to use a probability sequence to obtain the $T_{ac}$ most likely discrete Boolean variables, and the probabilities of the sequence follow the multiplication rule. For example, given the probability sequence $PB_i=[0.1, 0.2, 0.6]$, where the probability of the first symbol being predicted incorrectly is 0.1 (i.e., $b_{i1}$ is True with a probability of $0.1$ and False with a probability of 0.9), and so forth. The probability of $B_i$ taking the value [False, False, True] is the highest, at $0.9 * 0.8 * 0.6 = 0.432$. The probability of $B_i$ taking the value $[False, False, False]$ is the second highest, at $0.9 * 0.8 * 0.4 = 0.288$, and the probability of $B_i$ taking the value [False, True, False] is the third highest, at $0.9 * 0.2 * 0.6 = 0.108$. We want to select the $T_{ac}$ most probable Boolean sequences for querying the knowledge base. However, using the simplest approach, if the sequence length is $l$, computing all possible solutions' probabilities and sorting them to select the top $T_{ac}$ would have a complexity exceeding exponential time, $O(2^l \log2^l)$. We have designed an extremely rapid method that allows us to obtain the $T_{ac}$ most likely sequences based on the probability sequence in a very short time, requiring only $O(l \log l + T_{ac} \log T_{ac} +lT_{ac})$ computational complexity to achieve this goal.
\subsubsection{Initialize}
Our approach is as follows: Initially, the sequence with the highest probability can be directly obtained. We simply assign True to elements in the probability sequence greater than $0.5$, and False otherwise. Then, based on this initial sequence, we generate subsequent sequences. To simplify the description, for each position in a sequence, if it remains consistent with the initial state, we call it the original state; if it differs, we call it the flipped state. The transition from the initial state to the flipped state is referred to as flipping, while the transition from the flipped state back to the initial state is referred to as unflipping. Furthermore, we have proven a theorem:
\begin{theorem}
\label{theorem1}
Let \( a > 1 \). After identifying the probabilities of the top \( a-1 \) sequences, it holds that at least one sequence among the top \( a-1 \) solutions does not require any positions to transition from flipped to initial states in order to obtain the \( a \)-th top sequence. Additionally, only one position needs to transition from the initial to the flipped state.
\end{theorem}
\begin{proof}
\textbf{Part 1:}  
We first prove that among the top \( a-1 \) sequences, there exists at least one sequence that does not require any positions to transition from flipped to initial states in order to obtain the \( a \)-th top sequence. This is evident because the top sequence (i.e., the first sequence) is included among the top \( a \) sequences, and it does not require any transitions from flipped to initial states. Hence, the proof for this part is straightforward.

\textbf{Part 2:}  
Next, we prove that among the top \( a-1 \) sequences, there exists at least one sequence that requires only one position to transition from the initial to the flipped state to obtain the \( a \)-th top sequence.
Let \( a > b > 0 \), and suppose that the \( b \)-th sequence among the top \( a-1 \) sequences does not require any unflipping and requires the minimum number of flips, denoted by \( c \), where \( c > 1 \), to reach the \( a \)-th top sequence. Consider the case where one of the \( c \) flip positions is changed.
\begin{itemize}
  \item If the new sequence obtained is not among the top \( a-1 \) sequences, its probability must be greater than that of the \( a \)-th top sequence. This would imply that the current \( a \)-th top sequence is not truly the \( a \)-th top sequence, which contradicts our assumption.
  \item If the new sequence is among the top \( a-1 \), then the number of flips required to transition from this new sequence to the \( a \)-th top sequence is \( c-1 \). This contradicts the assumption that the \( b \)-th sequence requires the minimum number of flips, \( c \), without unflipping. 
\end{itemize}
Thus, by contradiction, we conclude that among the top \( a-1 \) sequences, there exists at least one sequence that requires only one position to transition from the initial to the flipped state in order to obtain the \( a \)-th top sequence.
\end{proof}
\subsubsection{Search Based on Max Heap}
According to the theorem above, we have already proven that the $a-th$ solutions can be obtained from one of the top $a-1$ solutions through a single flip, and in this process, all unflips are unnecessary. Hence, each flip only needs to choose the minimal-cost option that minimizes the total probability loss. For a position $u$, flipping reduces the total probability by a factor of:
$V_u = \frac{min(pb_u, 1 - pb_u)}{max(pb_u, 1 - pb_u)}$, where $V$ is fixed for each position. Therefore, we can pre-sort each position in descending order based on $V$ and simply choose the unflipped position with the highest rank during each flip. The sorting time complexity is $O(l \log l)$, as shown in \cref{Sort}. 
\begin{figure}[htbp]
\vskip 0.2in
\begin{center}
\centerline{\includegraphics[scale=0.7]{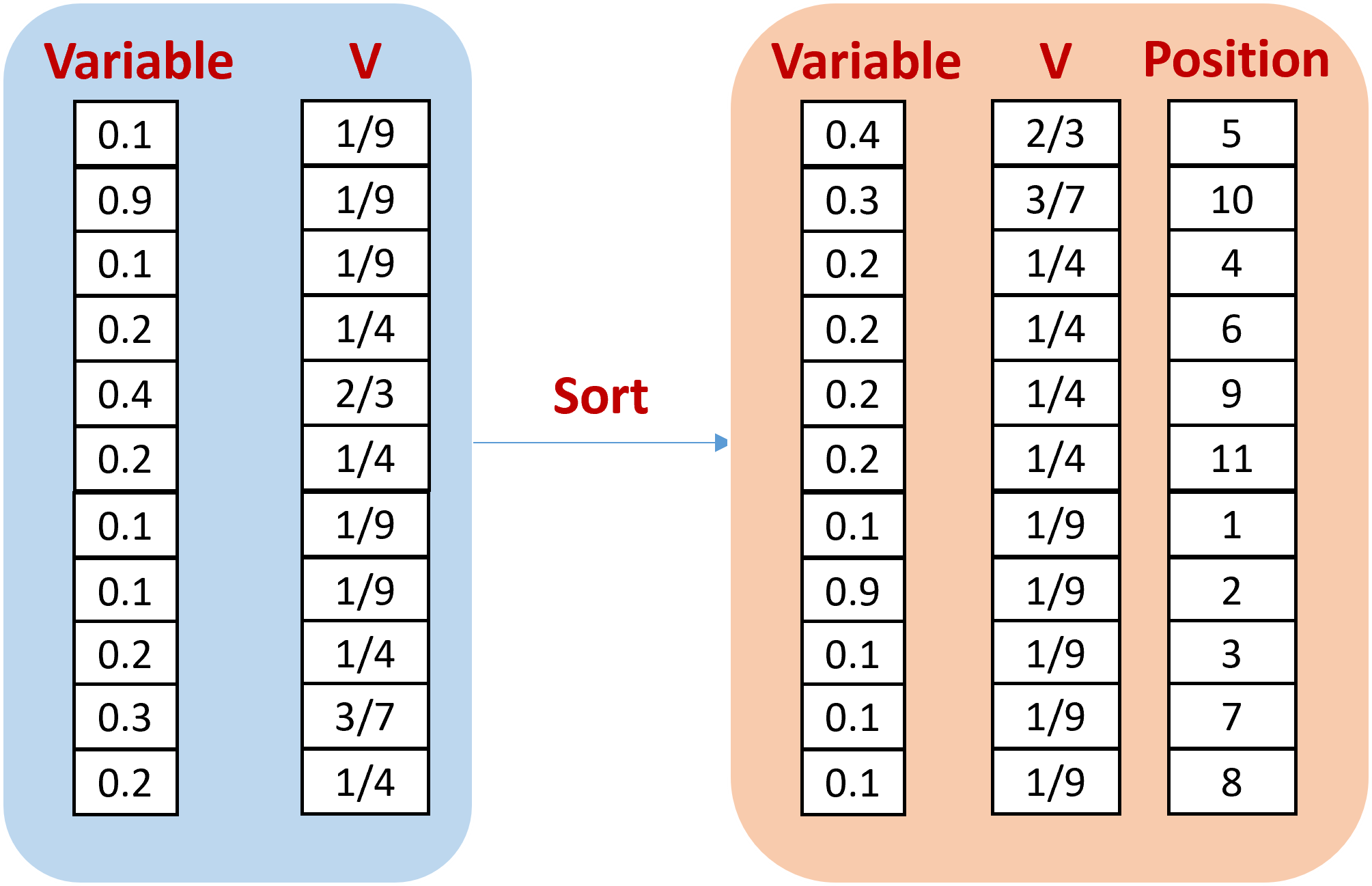}}
\caption{Ranking positions based on overall probability loss.}
\label{Sort}
\end{center}
\vskip -0.2in
\end{figure}
After obtaining a new sequence, we can easily determine its successor sequence obtained by flipping the least costly position and its total probability. We only need to maintain a max-heap as shown in \cref{Bridge} to store the top $a$ sequences found so far based on the probability of their next successor. This allows us to directly query the heap top to find the sequence with the highest successor probability with time complexity of $O(1)$. This successor has the highest probability among all sequences not selected. After identifying the top sequence in the heap, we need to locate the successor sequence of its current successor sequence. Subsequently, we store the current successor sequence into the heap. Upon finding the next successor sequence of the top sequence, we update its position in the heap. In the absence of conflicts, the total time complexity for $T_{ac}$ solutions is $O(2T_{ac}\log T_{ac})$,  as the time complexity for adjusting the position of a node in the heap once is $\log T_{ac}$.
\begin{figure*}[htbp]
\vskip 0.2in
\begin{center}
\centerline{\includegraphics[scale=0.32]{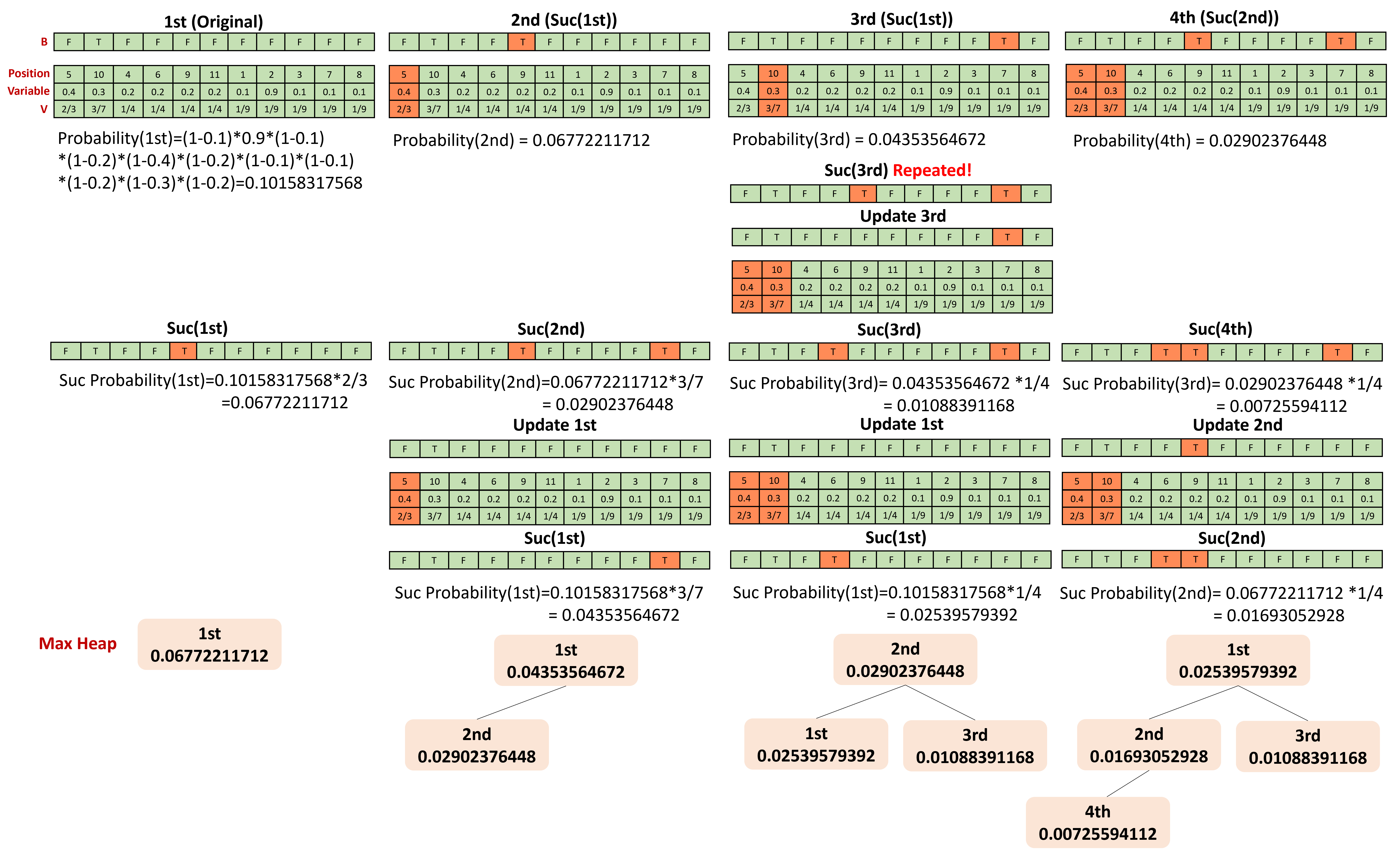}}
\caption{The schematic diagram of the search algorithm for finding the $T_{ac}$ discrete Boolean sequences $B$ with the highest probability based on continuous probability sequence $PB$.}
\label{Bridge}
\end{center}
\vskip -0.2in
\end{figure*}
\subsubsection{Handling Conflicts}
In some cases, the current successor sequence may duplicate the previously obtained successor sequence, such as [False, False, True] and [False, True, False], both resulting in [False, True, True]. When we encounter conflicts, the current sequence needs to skip its next successor sequence and instead find the successor sequence with a lower rank to avoid the conflict. In other words, rather than selecting the position with the minimum loss among the positions that can be flipped, we choose the second smallest position. 

Since each search for a new solution requires finding the next successor of the current top sequence as well as the successor of its current successor, at most $2l$ conflicts occur. The time required for conflict detection is $O(2lT_{ac})$. So the total time complexity is $O(l \log l + 2T_{ac} \log T_{ac} + 2lT_{ac})$, which can be settled as $O(l \log l + T_{ac} \log T_{ac} +lT_{ac})$.
\subsubsection{Feedback}
Through the aforementioned method, we efficiently obtain the most probable $T_{ac}$ sets of Boolean sequences. We then conduct knowledge base queries based on each of them to find the sequence with the highest consistency with the knowledge base, which serves as the final result of the inference. The logic-revised results are used to provide feedback to the machine learning perception model, and the final Boolean sequence is used as feedback for the BSNN. This approach allows for the accumulation of successful experiences from each inference process, enhancing both sample-level and symbol-level perception. Consequently, it aligns the model's perception results more closely with the knowledge base, subtly integrating knowledge into the machine learning model.
\section{Experiments} 
\subsection{About the Evaluations}
We evaluated the proposed algorithm through diverse approaches. Initially, in preliminary experiments, we used the overall accuracy of ABL on the target task as the performance metric and the number of accesses to the knowledge base $T_{ac}$ as the time metric. We demonstrated the efficiency of algorithms in two forms: comparing target performance with fixed time consumption and comparing time consumption under fixed target performance. However, we later found that this evaluation method is not sufficiently scientific. This is because in practical applications, knowledge bases are typically incomplete. The inputted constant symbols to the knowledge base may neither be provable nor disprovable. As a result, the performance of ABL is significantly affected by the accidental factor of whether the rules derived from the knowledge base are correct. This inner problem cannot be resolved by the outer optimization algorithm and can lead to learning failures even if the optimization module is optimized to the extreme. Moreover, in cases of poor knowledge base quality, worse performance on the target task can actually prove the superiority of the optimization algorithm. 

To address the fairness issue mentioned above, we propose two methods to independently and objectively evaluate the performance of optimization algorithms of ABL, excluding external accidental factors. The first method involves evaluating under the condition of having a complete knowledge base. However, this comparison method is utopian in practical applications. The second method, and the one we advocate the most, is to return to the original goal of the optimization algorithm: to integrate knowledge from the knowledge base into the machine learning model. Hence, what we should primarily evaluate is the degree of consistency between the knowledge base and the machine learning model under the influence of the optimization algorithm.
\subsection{Experimental Setup}
\label{Experimental Setup}
We conducted experiments on three datasets ~\cite{dai2019bridging}, Digital Binary Additive (DBA), Random Symbol Binary Additive (RBA) and Handwritten Math Symbols (HMS). The first two datasets are from ~\cite{dai2019bridging}, while HMS comes from ~\cite{licloses}, and is constructed in the same way as DBA and RBA in our experiments. Unlike previous works on ABL that typically impose restrictions on accessing the knowledge base to hundreds of times or more, we strictly limited the number of accesses to the knowledge base to ${T_{ac}}=5$. We set the length of individual sample sequences to be between 5 and 10. Additionally, we specified that every 3 sample sequences form a group for ABL. At the end of the ABL phase, similarly, we combine 3 sequences into one group and incorporate knowledge-based learning to derive a set of rules. Then, based on the satisfaction of each rule set by the sample sequences and the labels $Y$ indicating whether it conforms to logic, we utilize a Multi-Layer Perceptron (MLP) model for supervised training. This MLP model is eventually employed for predicting the test data combined with rule sets. We utilized LeNet5 ~\cite{lecun1998gradient} as the machine learning model for the perceptual part of entities and a bidirectional Long Short Term Memory Networks (LSTM) ~\cite{hochreiter1997long} model with a hidden layer dimension of 10 as the sequence neural network for the symbolic perception part. The length of the LSTM sequences is set to 10. For sequences shorter than 10, we pad them with leading zeros. At any stage, the number of epochs for all neural networks is set to 10. We set the total number of iterations for conducting ABL to be 150.

Before ABL begins, we employed self-supervised learning approaches to pretrain the neural networks. For LeNet5, it needs to be converted into an autoencoder before pretraining, where the intermediate layer dimension is set to $|SYM|$, representing the number of symbols. Since after pretraining, we do not know the mapping between each dimension of the intermediate layer and specific symbols, we choose to retrieve the knowledge base and select the mapping with the highest match. For LSTM, the pretraining method involves taking originally logically correct sequences and randomly replacing less than half of the samples. The LSTM then determines whether each position has been replaced or not. We constructed the knowledge base using the SWI-Prolog ~\cite{wielemaker2012swi} tool and implemented the deep learning code using the TensorFlow-based Keras framework ~\cite{abadi2016tensorflow}. All experiments were conducted on 4 A800 GPUs.
\subsection{Comparison Methods}
\label{Comparison Methods}
We compared our proposed optimization algorithm with five gradient-free optimization algorithms, namely RACOS, SRACOS, SSRACOS, POSS, PONSS and ABL-REFL which is an ABL paradigm that directly predicts the Boolean list, in the same framework. They are:
\begin{itemize}
\item \textbf{RACOS}: RACOS is a proposed classification-based derivative-free optimization algorithm. Unlike other derivative-free optimization algorithms, the sampling region of RACOS is learned by a simple classifier.
\item \textbf{SRACOS}: SRACOS is the sequential version of RACOS.
\item \textbf{SSRACOS}: SSRACOS is a noise handling variant of SRACOS. 
\item \textbf{POSS}: POSS is another derivative-free optimization approach that employs evolutionary Pareto optimization to find a small-sized subset with good performance.
\item \textbf{PONSS}: PONSS is a noise handling variant of POSS. 
\item \textbf{ABL-REFL}: ABL-REFL is a paradigm that directly utilizes neural networks to predict Boolean lists. 
\end{itemize}
\subsection{Experiments with Incomplete Knowledge Base}
Based on the analysis, the time consumed by ABL is primarily determined by the number of accesses $T_{ac}$ to the knowledge base. Therefore, there are two ways to evaluate algorithm efficiency: one is to compare performance with a fixed number of accesses, and the other is to compare the required number of accesses for fixed performance. 

For the former form, We primarily evaluate the final performance $ACC_{IT_{total}}$ after iteration completion, the best performance $ACC_{best}$ during iteration, and the final convergence rate $CR$, where $IT_{total}=150$ is the number of iterations. For the last form, we compare the number of accesses ${T_{ac}}_a$ needed for the algorithm to achieve and never fall below a fixed accuracy $a$ and the number of accesses ${T_{ac}}_c$ needed to achieve a fixed convergence rate $c$. The fixed accuracy $a$ is taken as $65\%$, while the fixed convergence rate $c$ is set at $3\%$ because most discrete optimization algorithms struggle to converge. We use `$-$' to indicate algorithms that ultimately fail to reach the objective, with the required number of accesses remaining unknown. The evaluation precision of all metrics will be influenced by the evaluation interval $\Delta_{IT}=10$ and the number of accesses $T_{ac}=5$ per iteration. In particular, the convergence rate $CR$ is defined as:
\begin{align*}
CR=\frac{|ACC_{IT_{total}}-ACC_{IT_{total}-\Delta_{IT}}|}{\Delta_{IT}}. 
\end{align*}
The variation of accuracy with the number of iterations on the three datasets can be seen in \cref{DBA_Incompleted,RBA_Incompleted,HMS_Incompleted} and the evaluation results are shown in \cref{DBA_T,RBA_T,HMS_T}.

\begin{table}[htbp]
    \centering
    \captionof{table}{The evaluation results on the DBA dataset.}
    \label{DBA_T}
    \begin{tabular}{c c c c c c}
    \hline\hline
    Method &$ACC_{IT_{total}}$&$ACC_{best}$&$CR$&${T_{ac}}_{a}$&${T_{ac}}_{c}$\\
        \hline\hline
        RACOS&61.22&73.56&0.29&$>750$&700\\
        SRACOS&77.33&77.33&2.46&750&$>750$\\
        SSRACOS&88.33&72.67&1.57&300&$>750$\\
        POSS&73.00&76.56&0.27&700&700\\
        PONSS&65.56&74.78&0.58&150&$>750$\\
        ABL-REFL&71.50&74.33&0.35&100&$>750$\\
        ABL-PSP&\textbf{96.67}&\textbf{96.67}&\textbf{0.00}&\textbf{50}&\textbf{550}\\
        \hline\hline
    \end{tabular}
\end{table}

\begin{figure}[htbp]
    \centering
    \centering
    \includegraphics[scale=1.0]{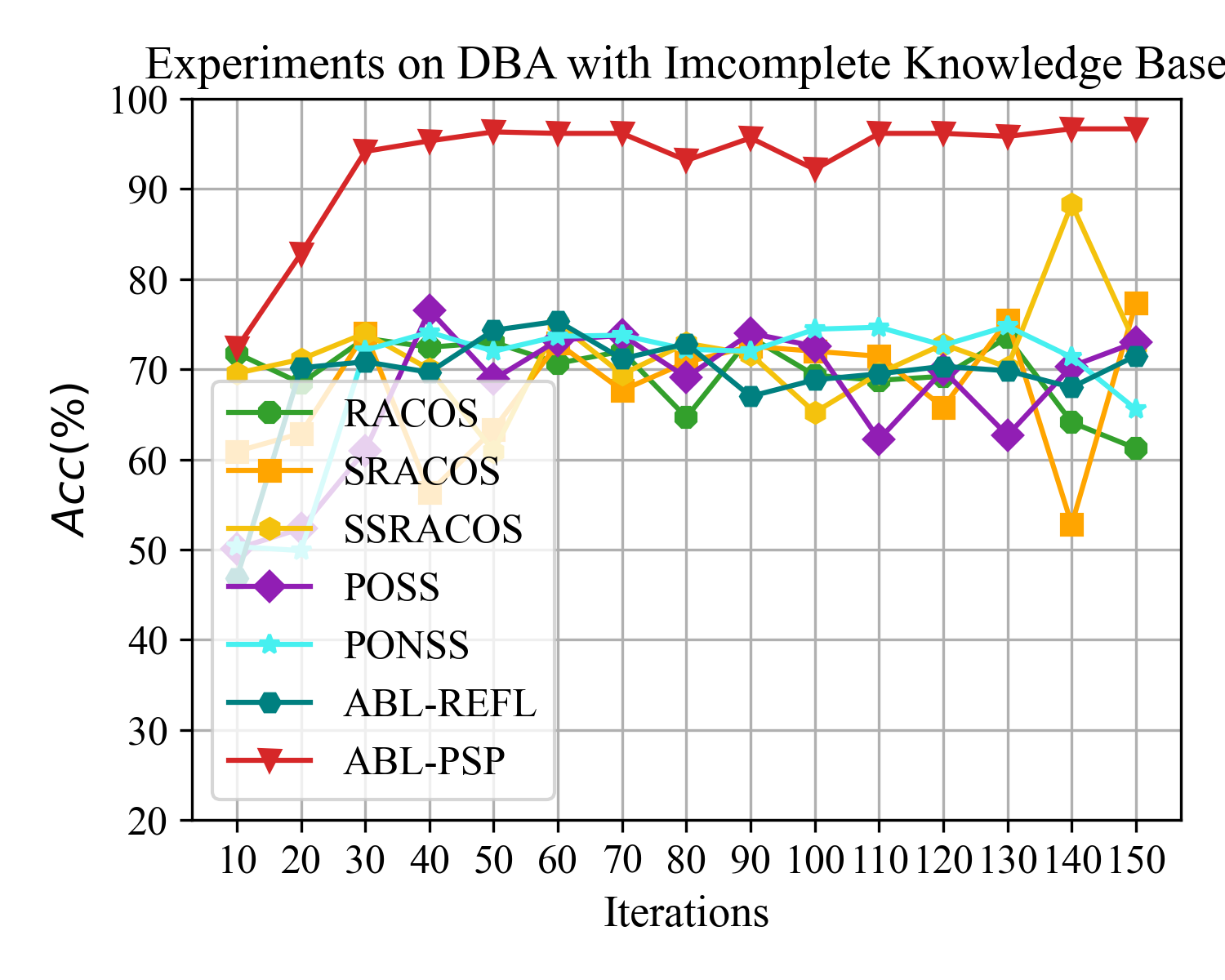} 
    \caption{The variation of accuracy on the DBA dataset with the incomplete knowledge base.}
    \label{DBA_Incompleted}
\end{figure}




\begin{table}[htbp]
\centering
\caption{The evaluation results on the RBA dataset.}
\label{RBA_T}
\vskip 0.15in
\begin{tabular}{c c c c c c}
\hline\hline
Method &$ACC_{IT_{total}}$&$ACC_{best}$&$CR$&${T_{ac}}_{a}$&${T_{ac}}_{c}$\\
    \hline\hline
    RACOS&62.17&67.17&0.53&$>750$&$>750$\\
    SRACOS&65.33&66.83&1.27&750&$>750$\\
    SSRACOS&49.11&68.33&1.09&$>750$&$>750$\\
    POSS&62.67&67.67&0.23&$>750$&600\\
    PONSS&68.67&69.33&0.38&750&$>750$\\
    ABL-REFL&63.67&69.50&0.40&$>750$&$>750$\\
    ABL-PSP&\textbf{71.50}&\textbf{71.50}&\textbf{0.05}&\textbf{300}&\textbf{500}\\
    \hline\hline
\end{tabular}
\end{table}
\begin{figure}[htbp]
\vskip 0.2in
\begin{center}
\centerline{\includegraphics[scale=1.0]{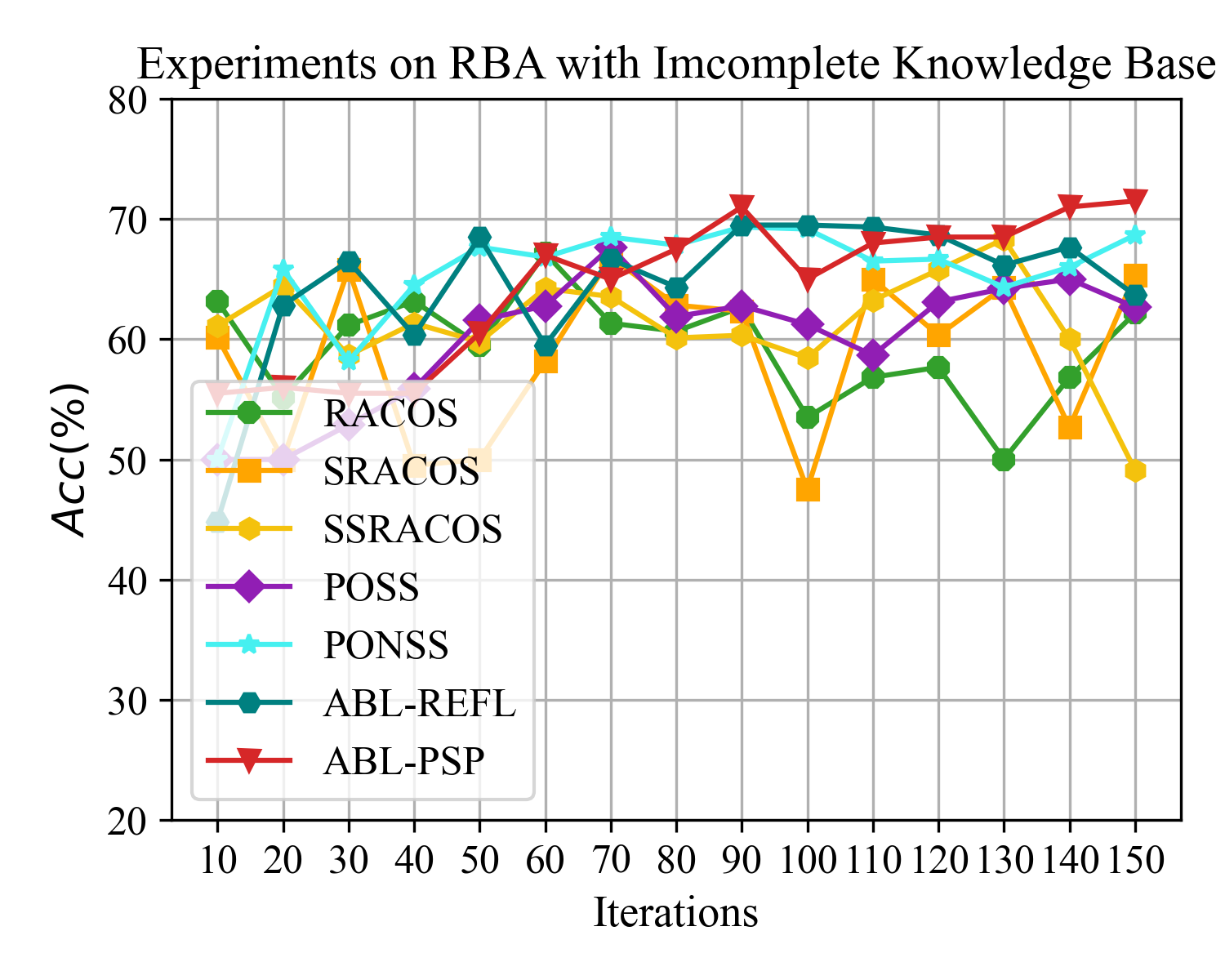}}
\caption{The variation of accuracy on the RBA dataset with the incomplete knowledge base.}
\label{RBA_Incompleted}
\end{center}
\vskip -0.2in
\end{figure}

\begin{table}[htbp]
\centering
\caption{The evaluation results on the HMS dataset.}
\label{HMS_T}
\vskip 0.15in
\begin{tabular}{c c c c c c}
\hline\hline
Method &$ACC_{IT_{total}}$&$ACC_{best}$&$CR$&${T_{ac}}_{a}$&${T_{ac}}_{c}$\\
    \hline\hline
    RACOS&61.50&68.67&1.02&$>750$&$>750$\\
    SRACOS&50.00&71.00&\textbf{0.05}&$>750$&650\\
    SSRACOS&69.00&70.33&1.90&750&$>750$\\
    POSS&71.00&71.00&0.15&700&700\\
    PONSS&50.00&\textbf{72.33}&0.08&$>750$&600\\
    ABL-REFL&71.33&\textbf{72.33}&\textbf{0.05}&700&700\\
    ABL-PSP&\textbf{72.00}&72.00&\textbf{0.05}&\textbf{200}&\textbf{200}\\
    \hline\hline
\end{tabular}
\end{table}
The experimental results indicate that ABL-PSP enhances the efficiency of ABL, demonstrating better performance compared to other optimization algorithms when the number of trial and error attempts is restricted. Based on the results, it is not difficult to find that, when the number of accesses is fixed, most optimization algorithms tend to perform relatively poorly. However, since ABL-PSP maximizes the utilization of all available information and past experience to reduce the uncertainty of solutions, it achieves decent results with very few attempts. When the learning objective is fixed, ABL-PSP stands out as one of the rare algorithms that achieve the target without the need for extensive trials, and the time saved by ABL-PSP is incalculable. 
\begin{figure}[htbp]
\vskip 0.2in
\begin{center}
\centerline{\includegraphics[scale=1.0]{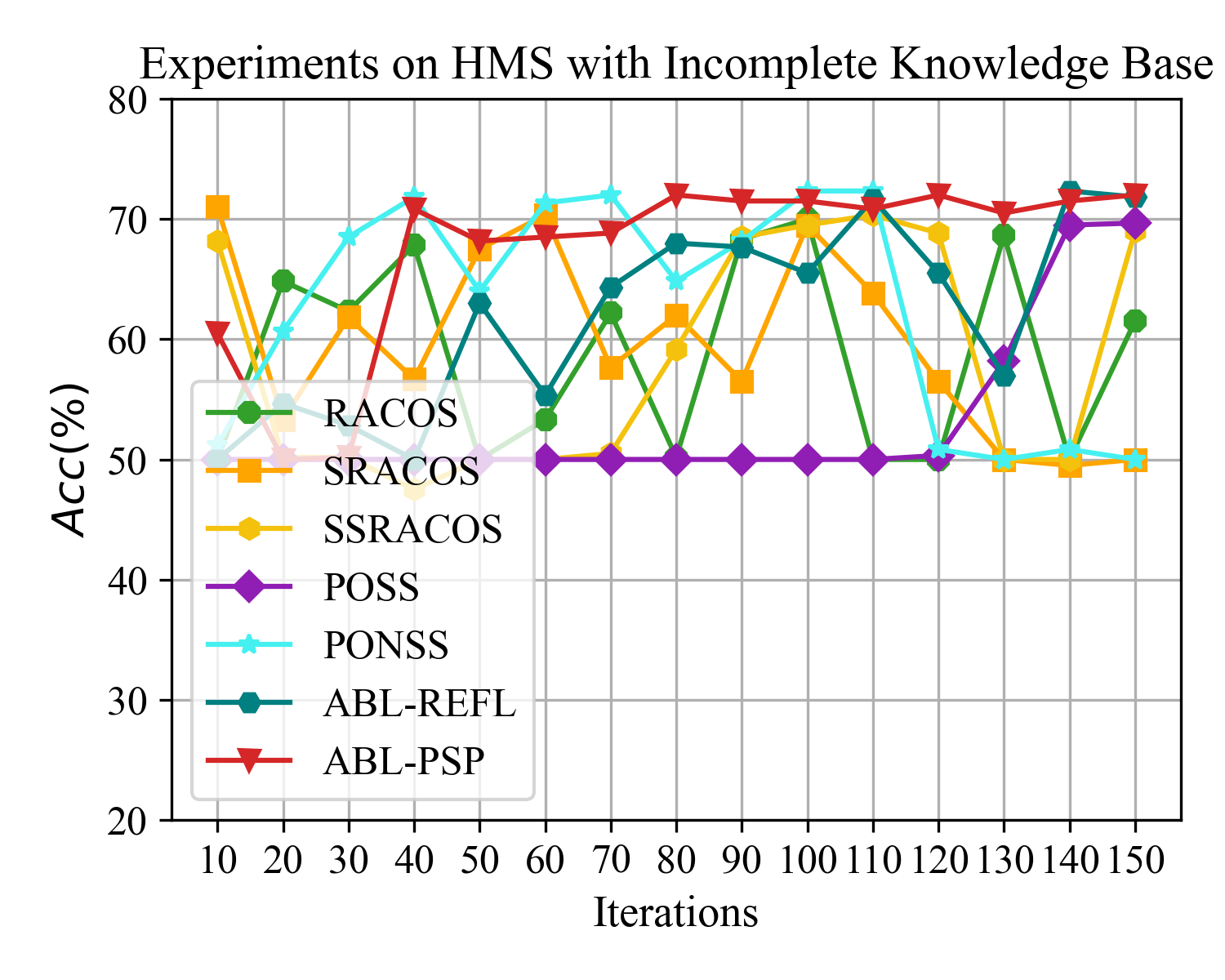}}
\caption{The variation of accuracy on the HMS dataset with the incomplete knowledge base.}
\label{HMS_Incompleted}
\end{center}
\vskip -0.2in
\end{figure}
\subsection{Experiments with Complete Knowledge Base}
Experimenting on the aforementioned incomplete knowledge base can partially reflect the performance ceiling of optimization algorithms when inference succeeds. However, it cannot be guaranteed that every inference will succeed in cases where the knowledge base is incomplete and labeled information is lacking. Previously, ABL used a validation set to determine whether to retrain the model midway through iterations, which is unfair to the evaluation of optimization algorithms\cite{dai2019bridging}. Therefore, we need to consider how to evaluate optimization algorithms under the premise of excluding interference from the knowledge base. Our first approach is to use a complete knowledge base for evaluation. This minimizes the influence of the knowledge base itself and the sequence of sample inputs on the evaluation process, thereby maintaining consistency between the algorithm's ability to incorporate the knowledge base into the ML model and its evaluation performance. Some related experiments have been conducted using such complete knowledge bases in prior work ~\cite{huang2021fast}, where all possible facts and conclusions are stored in the knowledge base, thereby eliminating the influence of randomness in the reasoning process on performance. However, this comparison method is utopian in practical applications due to substantial costs to construct a complete knowledge base, which is impossible on complex tasks. And if a complete knowledge base truly exists, then the optimization algorithm is no longer needed, and instead, just string matching algorithms would suffice, leading to a logical paradox: when we can evaluate the optimization algorithm, it means we no longer need it; when we need the optimization algorithm, it means we cannot evaluate it.

We have pre-built complete knowledge bases for the DBA, RBA, and HMS tasks respectively and then conducted the experiments. The experimental results are shown in \cref{DBA_Complete,RBA_Complete,HMS_Completed}.
\begin{figure}[htbp]
\vskip 0.2in
\begin{center}
\centerline{\includegraphics[scale=1.0]{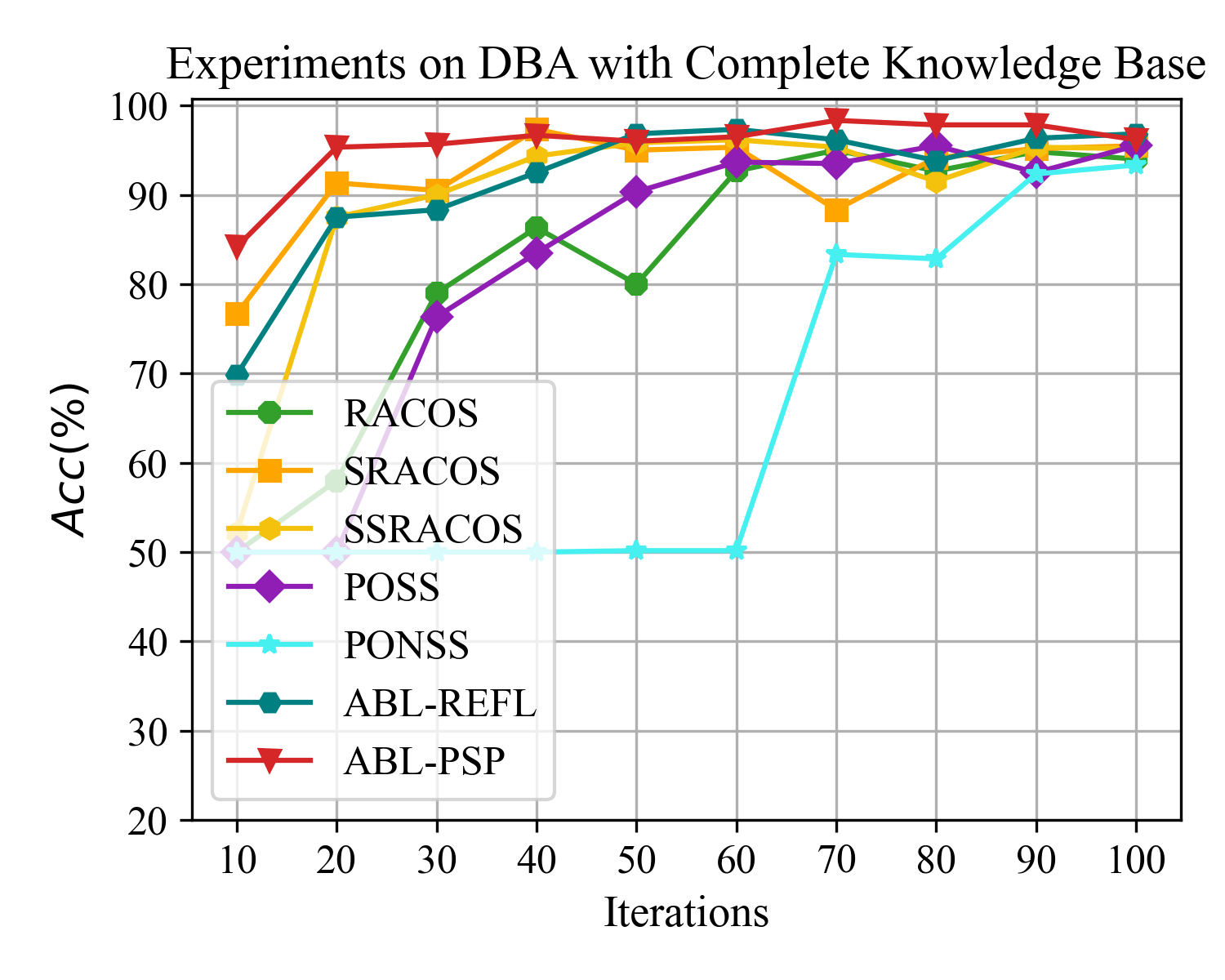}}
\caption{The variation of accuracy on the DBA dataset with the completed knowledge base.}
\label{DBA_Complete}
\end{center}
\vskip -0.2in
\end{figure}
\begin{figure}[htbp]
\vskip 0.2in
\begin{center}
\centerline{\includegraphics[scale=1.0]{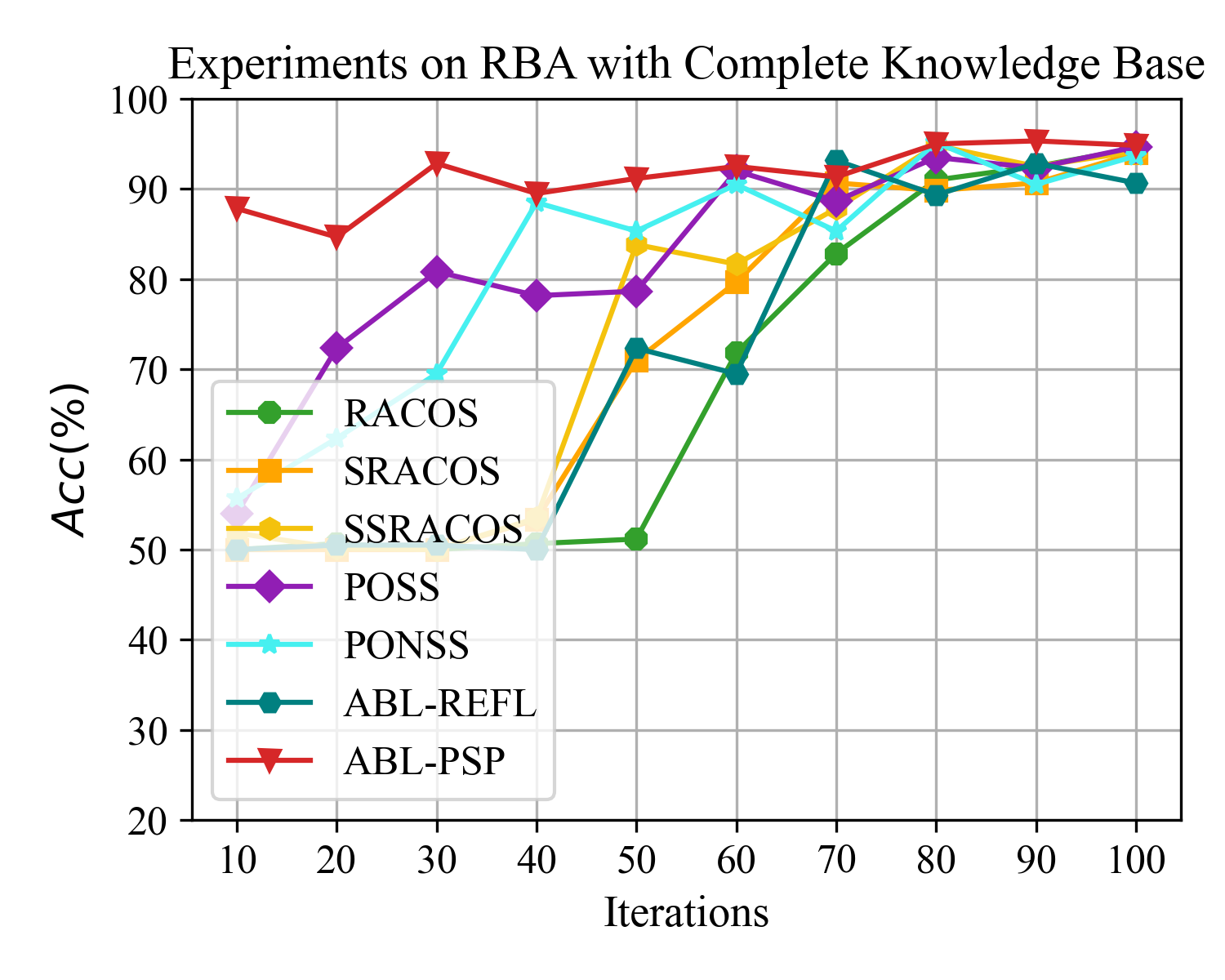}}
\caption{The variation of accuracy on the RBA dataset with the completed knowledge base.}
\label{RBA_Complete}
\end{center}
\vskip -0.2in
\end{figure}
\begin{figure}[htbp]
\vskip 0.2in
\begin{center}
\centerline{\includegraphics[scale=1.0]{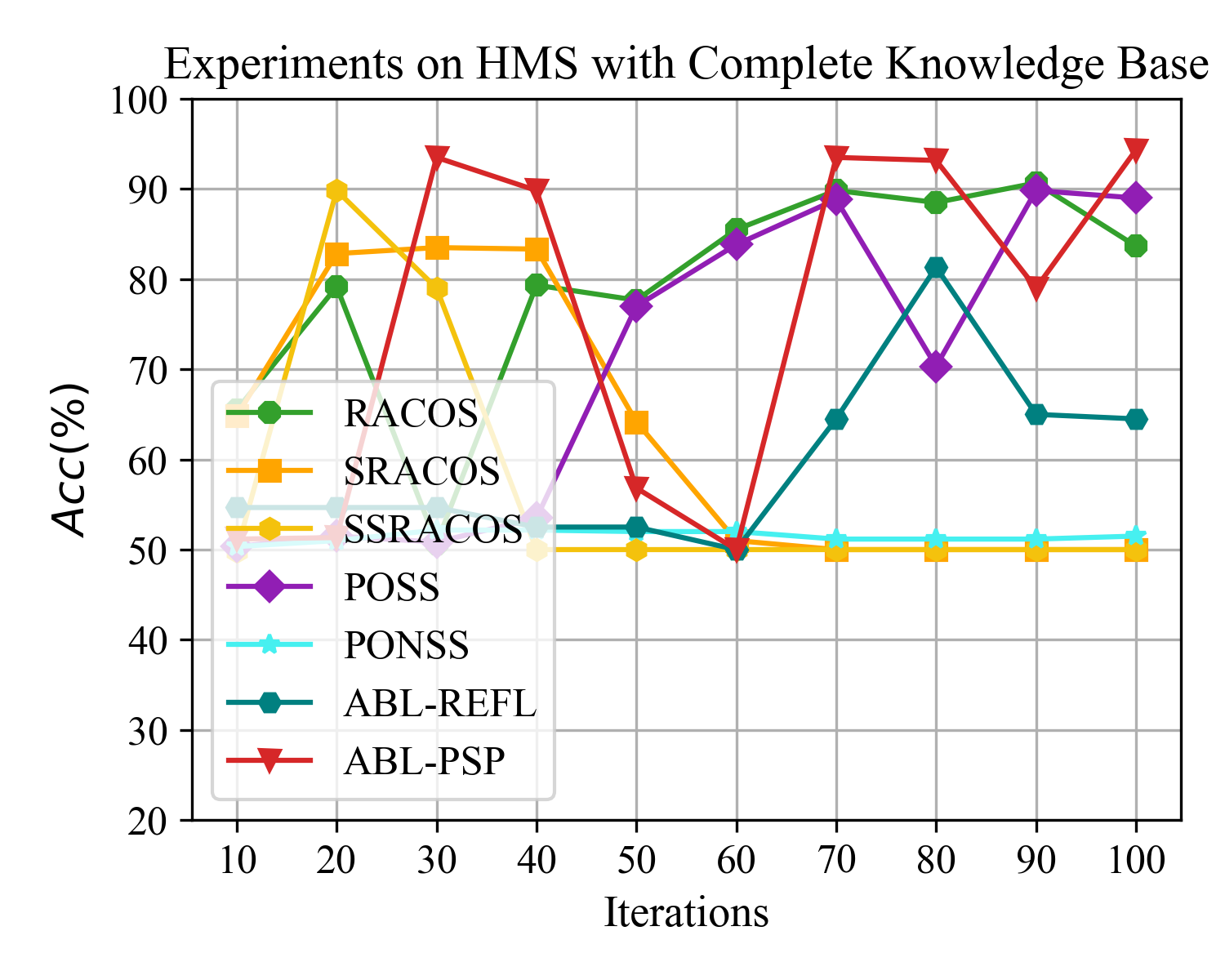}}
\caption{The variation of accuracy on the HMS dataset with the completed knowledge base.}
\label{HMS_Completed}
\end{center}
\vskip -0.2in
\end{figure}
\subsection{Experiments on Rule Generation}
While a complete knowledge base may alleviate some fairness issues, it is ineffective for most incomplete scenarios. Therefore, we have opted for new metrics that no longer measure the algorithm's merits solely based on final performance but rather assess how well optimization can align the ML model with the knowledge base. 
\begin{figure}[htbp]
\vskip 0.2in
\begin{center}
\centerline{\includegraphics[scale=1.0]{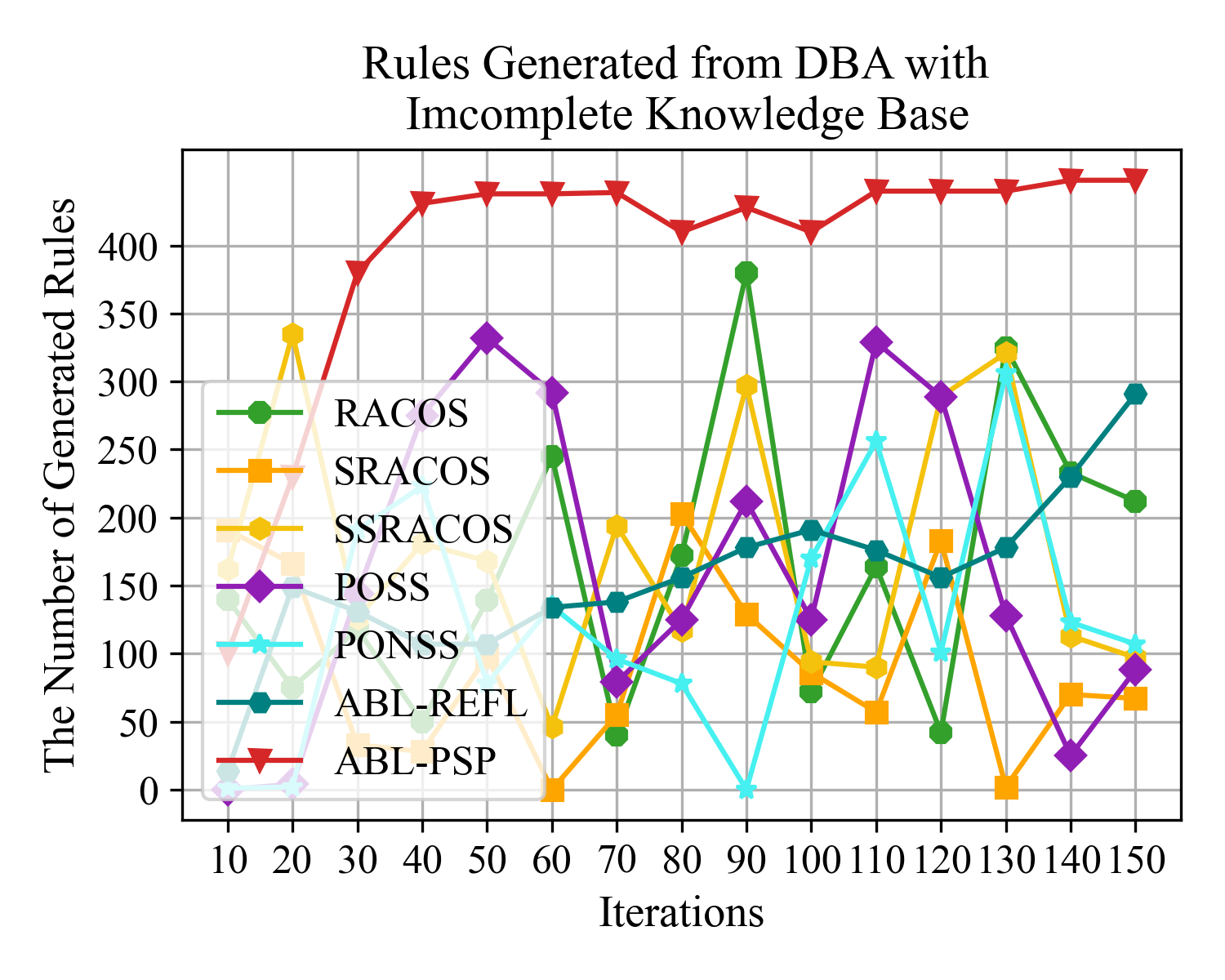}}
\caption{The variation of the number of generated rules on the DBA dataset with the incomplete knowledge base.}
\label{DBA_Incompleted_Rule}
\end{center}
\vskip -0.2in
\end{figure}

\begin{figure}[htbp]
\vskip 0.2in
\begin{center}
\centerline{\includegraphics[scale=1.0]{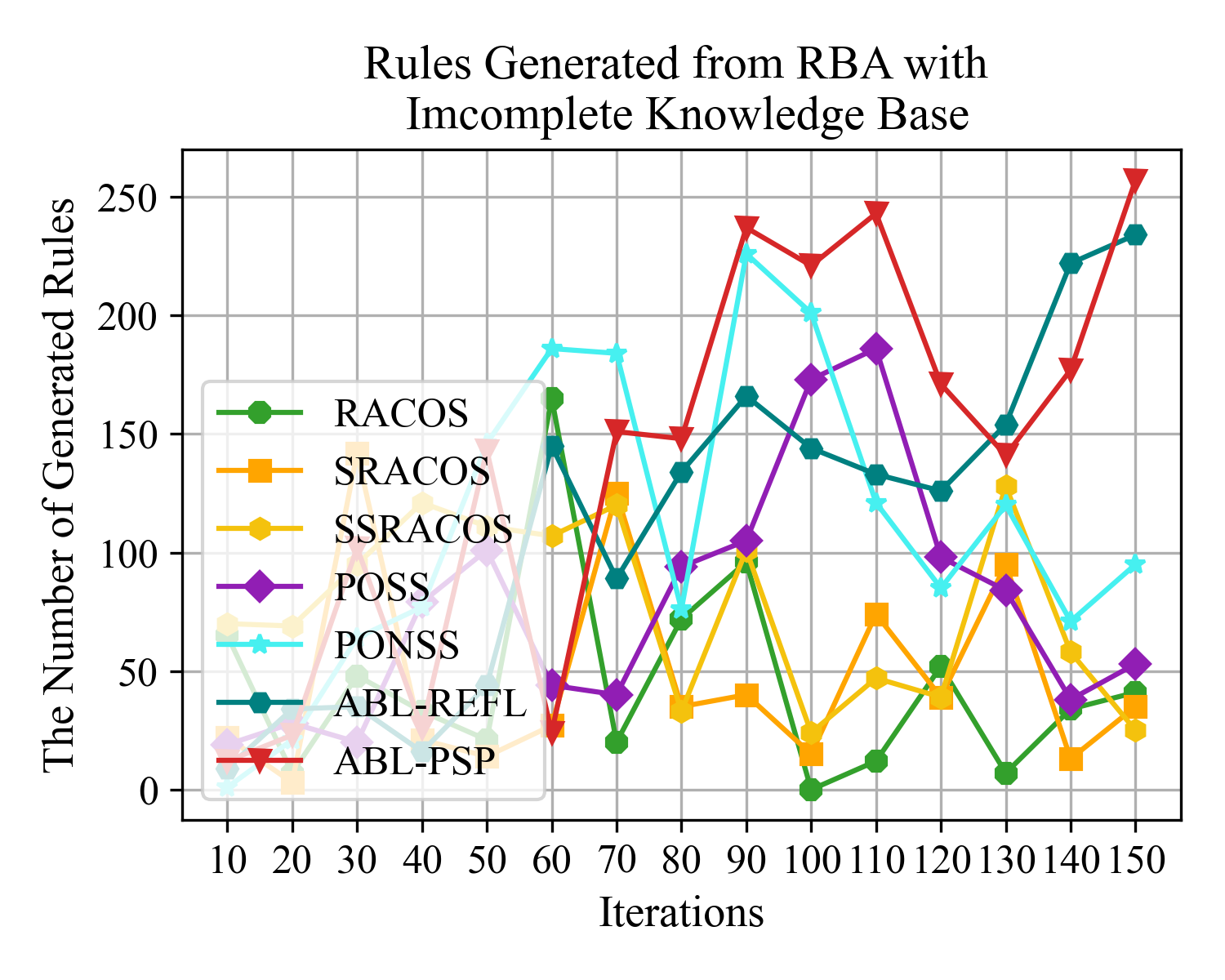}}
\caption{The variation of the number of generated rules on the RBA dataset with the incomplete knowledge base.}
\label{RBA_Incompleted_Rule}
\end{center}
\vskip -0.2in
\end{figure}

\begin{figure}[htbp]
\vskip 0.2in
\begin{center}
\centerline{\includegraphics[scale=1.0]{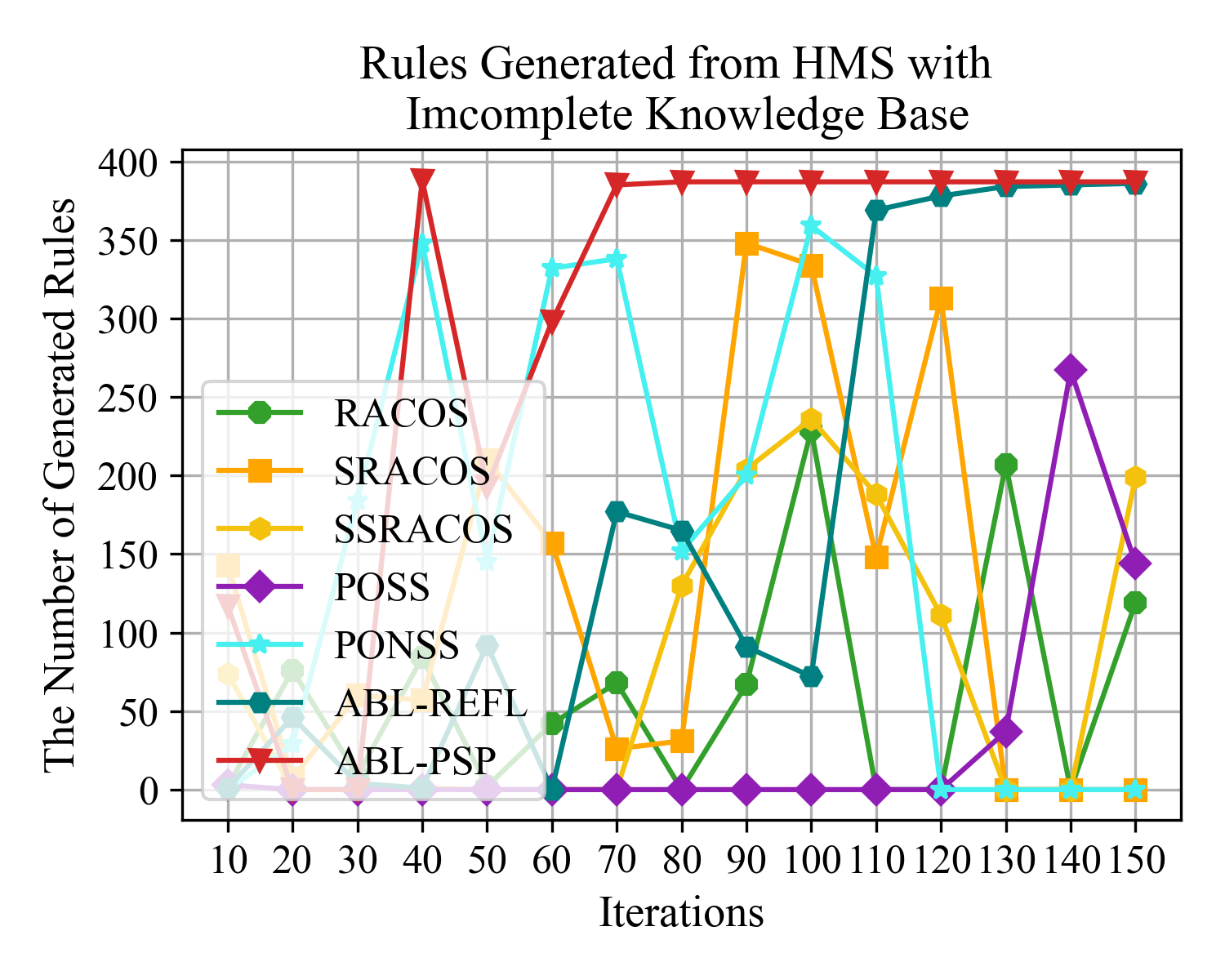}}
\caption{The variation of the number of generated rules on the HMS dataset with the incomplete knowledge base.}
\label{HMS_Incompleted_Rules}
\end{center}
\vskip -0.2in
\end{figure}
Hence, what we should primarily evaluate is the degree of consistency between the knowledge base and the machine learning model under the influence of the optimization algorithm. Specifically, after the ABL process concludes, we combine the data with the knowledge base in batches to form symbols through perception and generate new rules. We measure the degree to which the knowledge base is absorbed by the perception model by the number of successfully generated new rules. This is precisely the essence of optimization algorithms. 

We group every $s$ sample sequences and input their perceptual results into the knowledge base to count how many new rules can be generated, thereby evaluating the efficiency of the optimization algorithm. We conducted experiments with $s=3$ on 3 datasets.

This evaluation strategy allows the optimization algorithm in ABL to be evaluated independently of machine learning and logical reasoning and can serve as the most important and objective basis for evaluating the performance of optimization algorithms. The experimental results are shown in \cref{DBA_Incompleted_Rule,RBA_Incompleted_Rule,HMS_Incompleted_Rules}.

The results indicate that ABL-PSP, can integrate external knowledge into machine learning models at the fastest rate.
\section{Conclusion}
In this paper, we confront two major challenges in integrating perception and reasoning models: the problem of extensive trial and error and the issue of converting continuous and discrete variables. Building upon ABL, we propose a solution that alleviates past algorithmic deficiencies in perception information, symbol relationships, and experience accumulation through sequence-based symbol perception built upon sample perception. This approach endows ABL with meta-reasoning capability based on symbol sensitivity akin to human experts, achieving promising results with only a few trial-and-error attempts. Additionally, we introduce probability as a bridge between continuous and discrete variable conversion, presenting an efficient and rigorously complete algorithm for converting a continuous probability sequence into discrete Boolean sequences.

\Acknowledgements{}

\Supplements{Appendix A.}

\bibliography{scis_paper}

\end{document}